\documentclass[journal]{IEEEtran}
\usepackage{bm}
\usepackage{epsfig}
\usepackage{graphicx}
\usepackage{subfigure}
\usepackage{float}
\usepackage{cite}
\usepackage{url}
\usepackage{color}
\usepackage{balance}
\usepackage{mdwlist}
\usepackage{multirow}
\usepackage{threeparttable}
\usepackage{tablefootnote}
\usepackage{enumitem}
\usepackage{amsmath}
\usepackage{amssymb}
\usepackage{amsthm}
\usepackage{stmaryrd}
\usepackage{booktabs}
\usepackage{siunitx}
\usepackage{comment}
\usepackage[numbers,sort&compress]{natbib}
\usepackage{epsfig,amsmath,amsfonts}
\usepackage{amsfonts}

\newcommand{\argmin}{\mathop{\mathrm{argmin}}}
\newcommand{\argmax}{\mathop{\mathrm{argmax}}}
\newcommand{\Tr}{\mathop{\mathrm{Tr}}}
\newcommand{\hp}{\mathop{\mathrm{hp}}}
\def\N{\mathbb{N}}
\def\R{\mathbb{R}}

\def\E{\mathbb{E}}
\def\P{\mathbb{P}}
\def\U{\mathcal{U}}

\newtheorem{lemma}{Lemma}

\newtheorem{observation}{Observation}
\newtheorem*{example}{Example}
\makeatletter
\newif\if@restonecol
\makeatother

\usepackage[ruled,vlined]{algorithm2e} 
\usepackage{algpseudocode}
\usepackage{amsmath}

\newcommand{\ie}{i.e.}

\newcommand{\parm}{{\xi}}
\newcommand{\vecpar}{\boldsymbol{\parm}}

\newcommand{\uniGPC}{\phi }
\newcommand{\multiGPC}{\Psi }

\newcommand{\polyInd}{\alpha}
\newcommand{\basisInd}{\boldsymbol{\polyInd}}

\newcommand{\ten}[1]{\mathcal{#1}}

\newcommand{\mat}[1]{\mathbf{#1}}
\newcommand{\vect}[1]{\boldsymbol{#1}}
\newcommand{\reff}[1]{(\ref{#1})}

\RequirePackage{url}

\def\BibTeX{{\rm B\kern-.05em{\sc i\kern-.025em b}\kern-.08em
    T\kern-.1667em\lower.7ex\hbox{E}\kern-.125emX}}
\begin{document}

\title{High-Dimensional Uncertainty Quantification via Tensor Regression with Rank Determination and Adaptive Sampling}

\author{\IEEEauthorblockN{Zichang He and Zheng Zhang, \textit{Member, IEEE}}
\thanks{The preliminary results of this work were published in EPEPS 2020~\cite{he2020high}. 
This work was partly supported by NSF grants \#1763699 and \#1846476. 

Zichang He and Zheng Zhang are with Department of Electrical and Computer Engineering, University of California, Santa Barbara, CA 93106, USA (e-mails: zichanghe@ucsb.edu;  zhengzhang@ece.ucsb.edu).
}
}
\specialpapernotice{(Invited Paper)}


\maketitle
\begin{abstract}
Fabrication process variations can significantly influence the performance and yield of nano-scale electronic and photonic circuits. 
Stochastic spectral methods have achieved great success in quantifying the impact of process variations, but they suffer from the curse of dimensionality. 
Recently, low-rank tensor methods have been developed to mitigate this issue, but two fundamental challenges remain open: how to automatically determine the tensor rank and how to adaptively pick the informative simulation samples. 
This paper proposes a novel tensor regression method to address these two challenges. 
We use a $\ell_{q}/ \ell_{2}$ group-sparsity regularization to determine the tensor rank. The resulting optimization problem can be efficiently solved via an alternating minimization solver.
We also propose a two-stage adaptive sampling method to reduce the simulation cost. Our method considers both exploration and exploitation via the estimated Voronoi cell volume and nonlinearity measurement respectively. 
The proposed model is verified with synthetic and some realistic circuit benchmarks, on which our method can well capture the uncertainty caused by 19 to 100 random variables with only 100 to 600 simulation samples.
\end{abstract}

\section{Introduction}\label{sec:introduction}
Fabrication process variations (e.g., surface roughness of
interconnects and photonic waveguide, and random
doping effects of transistors) have been a major concern in nano-scale chip design. They can significantly influence chip performance and decrease product yield~\cite{boning2008variation}.
Monte Carlo (MC) is one of the most popular methods o quantify the chip performance under uncertainty,  but it requires a huge amount of computational cost~\cite{kalos2009monte}. 
Instead, stochastic spectral methods based on generalized polynomial chaos (gPC)~\cite{xiu2003modeling} offer efficient solutions for fast uncertainty quantification by approximating a real uncertain circuit variable as a linear combination of some stochastic basis functions~\cite{strunz2008stochastic,zhang2013stochastic,manfredi2014stochastic}. 
These techniques have been increasingly used in design automation~\cite{He2019ICCAD,ahadi2016sparse,manfredi2019rational,manfredi2015generalized,yucel2015me,wang2015bayesian,zhang2019efficient,cui2018stochastic}. The main challenge of the stochastic spectral method is the curse of dimensionality: the computational cost grows very fast as the number
of random parameters increases. In order to address this fundamental challenge, many high-dimensional solvers have been developed. The representative techniques include (but are not limited to) compressive sensing~\cite{li2010finding,cui2019high}, hyperbolic regression~\cite{ahadi2016hyperbolic}, analysis of variance (ANOVA)~\cite{ma2010adaptive,yang2012adaptive}, model order reduction~\cite{el2010variation}, and hierarchical modeling~\cite{zhang2014stochastic_cicc,zhang2014enabling}, and tensor methods~\cite{zhang2016big,zhang2014enabling}. 

The low-rank tensor approximation has shown promising performance in solving high-dimensional uncertainty quantification problems~\cite{zhang2016big_SPI,rai2014sparse,konakli2016polynomial,chevreuil2015least,nouy2017low,zhang2016big}. By low-rank tensor decomposition, one may reduce the number of unknown variables in uncertainty quantification to a linear function of the parameter dimensionality. However, there is a fundamental question: how can we determine the tensor rank and the associated model complexity? Because it is hard to exactly determine a tensor rank {\it a-priori}~\cite{hillar2013most}, existing methods often use a tensor rank pre-specified by the user or use a greedy method to update the tensor rank until convergence~\cite{zhang2016big,shi2019meta,tang2020rank}. These methods often offer inaccurate rank estimation and are complicated in computation.
Besides rank determination, another important question is: how can we adaptively add a few simulation samples to update the model with a low computation budget? This is very important in electronic and photonic design automation because obtaining each piece of data sample requires time-consuming device-level or circuit-level numerical simulations. 


\textbf{Paper contributions.} 
We propose a novel tensor regression method for high-dimensional uncertainty quantification. Tensor regression has been studied in machine learning and image data analysis~\cite{zhou2013tensor,kossaifi2020tensor,guo2011tensor}. 
There are some existing works of automatic rank determination~\cite{zhao2015bayesian,luan2019prediction,li2020evolutionary} and adaptive sampling~\cite{krishnamurthy2013low,he2020active} for tensor decomposition and completion. 
The Bayesian frameworks~\cite{guhaniyogi2017bayesian,yu2018tensor} can enable and guide the adaptive sampling procedure for tensor regression, but limit the model in the meanwhile.
Focusing on uncertainty quantification, there are few works about tensor regression and its automatic rank determination and adaptive sampling. 
The main contributions of this paper include:         
\begin{itemize}[leftmargin=*]
    \item We formulate high-dimensional uncertainty quantification as a tensor regression problem. We further propose a $\ell_{q}/ \ell_{2}$ group-sparsity regularization method to determine rank automatically. Based on variation equality, the tensor-structured regression problem can be efficiently solved via a block coordinate descent algorithm with an analytical solution in each subproblem.
    \item We propose a two-stage adaptive sampling method to reduce the simulation cost. This method balances the exploration and exploitation via combining the estimation of Voronoi cell volumes and the nonlinearity of an output function.  
    \item We verify the proposed uncertainty quantification model on a 100-dim synthetic function, a 19-dim photonic band-pass filter, and a 57-dim CMOS ring oscillator. Our model can well capture the high-dimensional stochastic output with only 100-600 samples. 
\end{itemize}
Compared with our conference paper~\cite{he2020high}, this manuscript presents the following additional results:
\begin{itemize}[leftmargin=*]
    \item The detailed implementations of the proposed method, including both the compact tensor regression solver and the adaptive sampling procedure (Section~\ref{sec:proposed_method} and~\ref{sec:adap_sampling})
    \item The post-processing step of extracting statistical information from the obtained tensor regression model (Section~\ref{sec:post-processing}).
    \item The enriched experiments (Section~\ref{sec:numerical_results}), including a demonstrative synthetic example and detailed comparisons with other methods. 
\end{itemize}

\section{Notation and Preliminaries}
\label{sec:preliminaries}

Throughout this paper, a scalar is represented by a lowercase letter, e.g., $x \in \mathbb{R}$; a vector or matrix is represented by a boldface lowercase or capital letter respectively, e.g., $\mat{x} \in \mathbb{R}^n$ and $\mat{X} \in \mathbb{R}^{m\times n}$. A tensor,
which describes a multidimensional data array, is represented
by a bold calligraphic letter, e.g., $\ten{X} \in \mathbb{R}^{n_1\times n_2 \cdots \times n_d}$. The $(i_1, i_2, \cdots, i_d)$-th data element of a tensor $\ten{X}$ is denoted as $x_{i_1 i_2 \cdots i_d}$. Obviously $\ten{X}$ reduces to a matrix $\mat{X}$ when $d=2$, and its data element is $x_{i_1i_2}$.
In this section, we will briefly introduce the background of generalized polynomial chaos (gPC) and tensor computation.

\subsection{Generalized Polynomial Chaos Expansion}
Let $\vecpar = \left[\parm_1,\ldots,\parm_d \right] \in \R^d$ be a random vector describing fabrication process variations with mutually independent components. We aim to estimate the interested performance metric $y(\vecpar)$ (e.g., chip frequency or power) under such uncertainty.
We assume that $y(\vecpar)$ has a finite variance under the process variations. 
A truncated gPC expansion approximates $y(\parm)$ as the summation of a series of orthornormal basis functions~\cite{xiu2003modeling}:
\begin{equation}
\label{eq:Multi_expansion}
    y(\vecpar)\approx \hat{y}(\vecpar) = \sum_{\basisInd \in \Theta} {c}_{\basisInd}\multiGPC_{\basisInd}(\vecpar),
\end{equation}
where $\basisInd \in \N^d$ is an index vector in the index set $\Theta$, ${c}_{\basisInd}$ is the coefficient, and $\multiGPC_{\basisInd}$ is a polynomial basis function of degree $|\basisInd|=\alpha_1 + \alpha_2+ \cdots +\alpha_d$. 
One of the most commonly used index set is the total degree one, which selects multivariate polynomials up to a total degree $p$, i.e., 
\begin{equation}\label{eq:total_scheme}
 \Theta = \{ \basisInd | \alpha_k \in \N, 0 \le \sum_{k=1}^d \alpha_k \le p \},
\end{equation}
leading to a total of $\frac{\left(d+p\right)!}{d!p!}$ terms of expansion.
Let $\uniGPC^{(k)}_{\alpha_k} (\xi_k)$ denote the order-${\alpha_k}$ univariate basis of the $k$-th random parameter $\xi_k$, 
the multivariate basis is constructed via taking the product of univariate orthornormal polynomial basis:
\begin{equation}
    \multiGPC_{\basisInd}(\vecpar)   =  \prod_{k=1}^d \uniGPC^{(k)}_{\alpha_k} (\xi_k). 
\end{equation}
Therefore, given the joint probability density function $\rho(\vecpar)$, the multivariate basis satisfies the orthornormal condition:
\begin{equation}
    \langle \multiGPC_{\basisInd}(\vecpar), \multiGPC_{\boldsymbol{\beta}}(\vecpar) \rangle= \int _{\mathbb{R}^d} \multiGPC_{\basisInd}(\vecpar) \multiGPC_{\boldsymbol{\beta}}(\vecpar) \rho(\vecpar) d \vecpar=\delta_{\basisInd,\boldsymbol{\beta}}.
\end{equation}
The detailed formulation and construction of univariate basis functions can be found in~\cite{xiu2003modeling,gautschi1982generating}.

In order to estimate the unknown coefficients ${c}_{\basisInd}$'s, several popular methods can be used, including intrusive (i.e., non-sampling) methods (e.g., stochastic Galerkin~\cite{ghanem1991stochastic} and stochastic testing~\cite{zhang2013stochastic}) and non-intrusive (i.e., sampling) methods (e.g.,  stochastic collocation based on pseudo-projection or regression~\cite{Xiu2016sc}). 
It is well known that gPC expansion suffers the curse of dimensionality. The computational cost grows exponentially as the dimension of $\vecpar$ increases. 

\subsection{Tensor and Tensor Decomposition}
Given two tensors $\ten{X}$ and $\ten{Y} \in \mathbb{R}^{n_1 \times n_2 \cdots \times n_d}$, their inner product is defined as:
\begin{equation}\label{eq:inner_pro}
    \langle \ten{X}, \ten{Y} \rangle := \sum_{i_1\cdots i_d} x_{i_1\cdots i_d} y_{i_1\cdots i_d}.
\end{equation}
A tensor $\ten{X}$ can be unfolded into a matrix along the $k$-th mode/dimension, denoted as ${\text{Unfold}_k}(\ten{X}):={\mat{X}_{(k)}} \in \mathbb{R}^{{n_k}\times {{n_1}\cdots{n_{k-1}}{n_{k+1}}\cdots{n_{d}} }}$.  
Conversely, folding the $k$-mode matrization back to the original tensor is denoted as ${\text{Fold}_k}({\mat{X}_{(k)}}):=\ten{X}$.

Given a $d$-dim tensor, it can be factorized as a summation some rank-1 vectors, which is called CANDECOMP/PARAFAC (CP) decomposition~\cite{kolda2009tensor}:
\begin{equation}\label{eq:cp}
    \ten{X}=\sum_{r=1}^R \mat{a}_r^{(1)} \circ \mat{a}_r^{(2)} \cdots \circ \mat{a}_r^{(d)} = [\![\mat{A}^{(1)}, \mat{A}^{(2)}, \ldots, \mat{A}^{(d)}]\!],
\end{equation}
where $\circ $ denotes the outer product. 
The last term is the Krusal form, where factor matrix $\mat{A}^{(k)} = \left[\mat{a}_1^{(k)}, \ldots, \mat{a}_R^{(k)} \right] \in \mathbb{R}^{n_k \times R}$ includes all vectors associated with the $k$-th dimension. The smallest number of $R$ that ensures the above equality is called a CP rank. 
The $k$-th mode unfolding matrix ${\mat{X}_{(k)}}$ can be written with CP factors as 
\begin{equation}\label{eq:minus_KR_product}
\begin{aligned}
    {\mat{X}_{(k)}} =& \mat{A}^{(k)} {\mat{A}^{(\setminus k)}}^T \; \text{with}\\
    \mat{A}^{(\setminus k)} = &\mat{A}^{(d)} \odot \cdots \odot \mat{A}^{(k-1)} \odot \mat{A}^{(k+1)} \cdots \odot \mat{A}^{(1)},
\end{aligned}
\end{equation}
where $\odot$ denotes the Khatri-Rao product, which performs column-wise Kronecker products~\cite{kolda2009tensor}.
More details of tensor operations can be found in~\cite{kolda2009tensor}.


\section{Proposed Tensor Regression method}
\label{sec:proposed_method}
\subsection{Low-Rank Tensor Regression Formulation}
To approximate $y(\vecpar)$ as a tensor regression model, we choose a full tensor-product index set for the gPC expansion:
\begin{equation}
\label{eq:indesSet}
    \Theta=\left\{ \basisInd=[\alpha_1, \alpha_2, \cdots, \alpha_d]\; |\; 0\leq \alpha_k \leq p , \forall k\in [1,d]\right\}.
\end{equation}
This specifies a gPC expansion with ${(p+1)^d}$ basis functions. 
Let $i_k=\alpha_k+1$, then we can define two $d$-dimensional tensors $\ten{X}$ and $\ten{B}(\vecpar)$ with their $(i_1, i_2, \cdots i_d)$-th elements as
\begin{align}
\label{eq:tensorDef}
x_{i_1 i_2 \cdots i_d}=c_{\basisInd} \; {\text{and}}\; b_{i_1 i_2 \cdots i_d}(\vecpar)=\multiGPC_{\basisInd}(\vecpar).
\end{align}
Combining Eqs~\eqref{eq:Multi_expansion},~\eqref{eq:indesSet} and~\eqref{eq:tensorDef}, the truncated gPC expansion can be written as a tensor inner product
\begin{equation}\label{eq:constructed_surrogate}
   y(\vecpar) \approx \hat{y}(\vecpar) = \langle \ten{X}, \ten{B}(\vecpar) \rangle.
\end{equation}
The tensor $\ten{B}(\vecpar) \in \R^{(p+1)\times \cdots \times (p+1)}$ is a rank-1 tensor that can be exactly represented as: 
\begin{equation}
   \ten{B(\vecpar)} =  \boldsymbol{\uniGPC}^{(1)} (\xi_1) \circ \boldsymbol{\uniGPC}^{(2)} (\xi_2) \circ \cdots \circ \boldsymbol{\uniGPC}^{(d)} (\xi_d),
\end{equation}
where $\boldsymbol{\uniGPC}^{(k)} (\xi_k)=[{\uniGPC}_0^{(k)} (\xi_k), \cdots, {\uniGPC}_p^{(k)} (\xi_k)]^T \in \mathbb{R}^{p+1}$ collects all univariate basis functions of random parameter $\xi_k$ up to order-$p$.

The unknown coefficient tensor $\ten{X}$ has ${(p+1)^d}$ variables in total, but we can describe it via a rank-$R$ CP approximation:
\begin{equation}\label{eq:coe_ten}
    \ten{X} \approx \sum_{r=1}^R \mat{u}_r^{(1)} \circ \mat{u}_r^{(2)} \circ \cdots \circ \mat{u}_r^{(d)} = [\![\mat{U}^{(1)}, \mat{U}^{(2)}, \ldots, \mat{U}^{(d)}]\!].
\end{equation}
It decreases the number of unknown variables to $(p+1)dR$, which only linearly depends on $d$ and thus effectively overcomes the curse of dimensionality. 

Our goal is to compute coefficient tensor $\ten{X}$ given a set of data samples $\left \{\vecpar_n, y(\vecpar_n)\right \}_{n=1}^N$ via solving the following optimization problem
\begin{equation}
\label{eq:reg_h}
\resizebox{\columnwidth}{!}{
    $\min \limits_{\{ \mat{U}^{(k)}\}_{k=1}^d} h(\ten{X})= \frac{1}{2} \sum \limits_{n=1}^N \left( y_n-\langle [\![\mat{U}^{(1)}, \mat{U}^{(2)}, \ldots, \mat{U}^{(d)}]\!], \ten{B}^n \rangle \right )^2,$
    }
\end{equation}
where $y_n=y(\vecpar_n)$, $\ten{B}^n=\ten{B}(\vecpar_n)$, and $\vecpar_n$ denotes the $n$-th sample. 

\subsection{Automatic Rank Determination}
\begin{figure}[t]
    \centering
    \includegraphics[width=1\columnwidth]{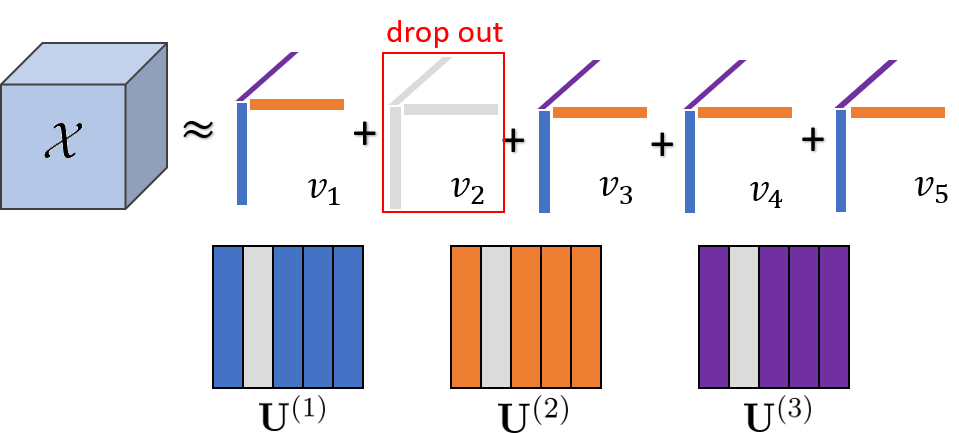}
    \caption{Visualization of the tensor rank determination. Here the gray vectors denote some shrinking tensor factors that can be removed from a CP decomposition.}
    \label{fig:rank_determination}
\end{figure}
The low-rank approximation~\reff{eq:coe_ten} assumes that $\ten{X}$ can be well approximated by $R$ rank-$1$ terms. In practice, it is hard to determine $R$ in advance. 
In this work, we leverage a group-sparsity regularization function to shrink the tensor rank from an initial estimation. Specifically, define the following vector:
\begin{equation}\label{eq:v_in_lq}
\resizebox{\columnwidth}{!}{
   $\mat{v}:=[v_1, v_2, \cdots, v_R]\, {\rm with}\; v_r= \left( \sum \limits_{k=1}^d \| \mat{u}_{r}^{(k)} \|_2^2 \right)^{\frac{1}{2}} \forall r\in [1,R].$
   }
\end{equation}
We further use its $\ell_{q}$ norm with $q \in \left( 0, 1\right]$ to measure the sparsity of $\mat{v}$: 
\begin{equation}
\label{eq:rank_penalty}
   g(\ten{X}) = \|\mat{v}\|_q, 
   \quad q \in \left( 0, 1\right].
\end{equation}
This function groups all rank-$1$ term factors together and enforces the sparsity among $R$ groups. The rank is reduced when the $r$-th columns of all factor matrices are enforced to zero. When $q=1$, this method degenerates to a group lasso, and a smaller $q$ leads to a stronger shrinkage force. 

Based on this rank-shrinkage function, we compute the tensor-structured gPC coefficients by solving a regularized tensor regression problem:
\begin{equation}\label{eq:tensor_regression}
\begin{aligned}
        \min_{\{\mat{U}^{(k)}\}_{k=1}^d} f(\ten{X}) = &  h(\ten{X})
        +  \lambda g(\ten{X}),
\end{aligned}
\end{equation}
where $\lambda>0$ is a regularization parameter.
As shown in Fig.~\ref{fig:rank_determination}, after solving this optimization problem, some columns with the same column indices among all matrices $\mat{U}^{(k)}$'s are close to zero. These columns can be deleted and the actual rank of our obtained tensor becomes $\hat{R}\leq R$, where $\hat{R}$ is the number of remaining columns in each factor matrix. 
\subsection{A More Tractable Regularization}
It is non-trivial to minimize $f(\ten{X})$
since $g(\ten{X})$ is non-differentiable and non-convex with respect to $\mat{U}^{(k)}$'s. 
Therefore, we replace the regularization function with a more tractable one based on the following variational equality. 
\begin{lemma}[Variational equality~\cite{jenatton2010structured}] 
\label{lemma:variational}
Let $\alpha \in (0, 2]$, and $\beta = \frac{\alpha}{2-\alpha}$. For any vector $\mat{y}\in \R^{p}$, we have the following equality 
\begin{equation}\label{eq:lemma_variantional}
    \|\mat{y}\|_\alpha = \min_{\boldsymbol{\eta} \in \R_{+}^p} \frac{1}{2}\sum_{r=1}^p \frac{y^2_r}{\eta_r} + \frac{1}{2}\|\boldsymbol{\eta}\|_\beta,
\end{equation}
where the minimum is uniquely attained for $\eta_r = {|y_r|}^{2-\alpha} \|\mat{y}\|_{\alpha}^{\alpha-1}, r=1,2,\ldots,p.$
\end{lemma}
\begin{proof}
See Appendix~\ref{sec:variational_ineq}.
\end{proof}
If we take $p=R$, $\alpha=q$, and ${y_r} = {v_r} $ (defined in $\eqref{eq:v_in_lq}$), on the right-hand side of Eq.~\reff{eq:lemma_variantional}, then we have 
\begin{equation}
        \hat{g}(\ten{X},\boldsymbol{\eta}) = \frac{1}{2}\sum \limits_{r=1}^R \frac{v_r^2}{{\eta}_r} + \frac{1}{2}\|\boldsymbol{\eta}\|_{\frac{q}{2-q}}.
\end{equation}
The original rank-shrinking function \eqref{eq:rank_penalty} is equivalent to
\begin{align}
\label{eq:subproblem_eta}
     g(\ten{X})  = \min_{\boldsymbol{\eta} \in \R_{+}^R} \hat{g}(\ten{X},\boldsymbol{\eta}).
\end{align}
As a result, we solve the following optimization problem as an alternative to~\eqref{eq:tensor_regression}:
\begin{equation}
    \label{eq:tensor_regression_bound}
\begin{aligned}
        \min \limits_{\{\mat{U}^{(k)}\}_{k=1}^d, \vect{\eta}} \hat{f}(\ten{X}) = & h(\ten{X}) + \lambda \hat{g}(\ten{X},\vect{\eta}).
\end{aligned}
\end{equation}


\subsection{A Block Coordinate Descent Solver for Problem~\reff{eq:tensor_regression_bound}}
Now we present an alternating minimization solver for Problem~\reff{eq:tensor_regression_bound}. Specifically, we decompose Problem~\reff{eq:tensor_regression_bound} into $(d+1)$ sub-problems with respect to $\{\mat{U}^{(k)}\}_{k=1}^d$ and $\vect{\eta}$, and we can obtain the analytical solution to each sub-problem. 
\begin{itemize}
 \item 
 \textbf{$\mat{U}^{(k)}$-subproblem:} By fixing $\vect{\eta}$ and all tensor factors except $\mat{U}^{(k)}$,
we can see that the variational equality induces a convex subproblem.
Based on the first-order optimality condition, $\mat{U}^{(k)}$ can be updated analytically:
\begin{equation}\label{eq:sol_sub_U}
    \text{vec}(\mat{U}^{(k)}) = {(\mat{\Phi}^T\mat{\Phi} + \lambda \tilde{\mat{\Lambda}})}^{-1} \mat{\Phi}^T \vect{y},
\end{equation}
where $\mat{\Phi} = {[\vect{\Phi}_1, \cdots, \vect{\Phi}_N]}^T$ $\in \R^{N \times R(p+1)}$ with rows  $\vect{\Phi}_n^T = \text{vec}\left({\mat{B}}_{(k)}^n \mat{U}^{(\setminus k)}\right)^T $ for any $ n\in [1,N] $, $\mat{\tilde{\Lambda}} = \text{diag}(\frac{1}{{\eta}_1},\ldots,\frac{1}{{\eta}_{R}}) \	  \otimes \mat{I} \in \R^{R(p+1) \times R(p+1)}$, and $\mat{y} \in \R^N$ is a collection of output simulation samples. Here $\otimes$ denotes a Kronecker product, $\mat{U}^{(\setminus k)}$ is a series of Khatri-Rao product defined as Eq.~\reff{eq:minus_KR_product}, and $\mat{B}_{(k)}^n$ is the $k$-th mode matrization of the tensor $\ten{B}^n=\ten{B}(\vecpar_n)$. 
For simplicity, we leave the derivation of Eq.~\reff{eq:sol_sub_U} to Appendix~\ref{sec:BCD_solution}. 
 \item
\textbf{$\boldsymbol{\eta}$-subproblem:} Suppose that $\{\mat{U}^{(k)}\}_{k=1}^d$ are fixed, the formulation of the $\boldsymbol{\eta}$-subproblem is shown in \eqref{eq:subproblem_eta}. According to lemma~\ref{lemma:variational}, we update $\vect{\eta}$ as 
\begin{equation}\label{eq:sol_sub_eta}
    {\eta}_r = ({{v}_r})^{2-q}\| \mat{v} \|_{q}^{q-1} + \epsilon,
\end{equation}
where $\epsilon > 0$ is a small scalar to avoid numerical issues. Suppose that the tensor rank is reduced in the optimization, \ie, $\mat{u}_{r}^{(k)} = \mat{0},~\forall k \in [1,d]$, then we can see that ${\eta_r}$ will become zero without $\epsilon$.
\end{itemize}

\subsection{Discussions}
We would like to highlight a few key points in practical implementations. 
\begin{itemize}[leftmargin=*]
\item The solution depends on the initialization process. In the first iteration of updating the $k$-th factor matrices, we suggest the following initialization 
\begin{equation}\label{eq:ini_of_U}
\begin{aligned}
  &\mat{\Phi}  = {[\vect{\Phi}_1, \vect{\Phi}_2, \ldots \vect{\Phi}_N ]}^T \; \text{with} \\
  &\vect{\Phi}_n = \text{vec}(\mat{O}_{n,k} )^T, \forall n \in [1,N], \\
  &\mat{O}_{n,k} =  
 \left[\vect{\uniGPC}^{(k)}(\xi_k^n), \ldots, \vect{\uniGPC}^{(k)}(\xi_k^n)\right]\in \R^{(p+1) \times R}, 
\end{aligned}
\end{equation}
where $\xi_k^n$ is the $k$-th variable of sample $\vecpar_n$, $\vect{\uniGPC}^{(k)}(\xi_k^n) \in \mathbb{R}^{p+1}$ collects all univariate basis functions of  $\xi_k^n$ up to degree $p$, and $\mat{O}_{n,k}$ stores $R$ copies of $\vect{\uniGPC}^{(k)}(\xi_k^n)$. 
Besides, in the first iteration without adaptive sampling, we set $\boldsymbol{\eta}$ as an all-ones vector multiplied by a scalar factor since we do not have a good initial guess for $\{{\mat{U}^{(k)}}\}_{k=1}^d$. The value of the scalar factor does not influence a lot once it makes Eq.~\eqref{eq:sol_sub_U} numerically stable.
In an adaptive sampling setting (see Section~\ref{sec:adap_sampling}), we need to solve~\reff{eq:tensor_regression_bound} after adding new samples. In this case, we use a warm-up initialization by setting the initial guess of $\{{\mat{U}^{(k)}}\}_{k=1}^d$ as the solution obtained based on the last-round sampling, and therefore can initialize $\boldsymbol{\eta}$ via Eq.~\eqref{eq:sol_sub_eta}. 

\item The regularization parameter $\lambda$ is highly related to the force of rank shrinkage. To adaptively balance the empirical loss and  the rank shrinkage term, we suggest an iterative update of the parameter 
\begin{equation}\label{eq:update_lambda}
    \lambda = \lambda_0 \max(\boldsymbol{\eta}), 
\end{equation}
where $\lambda_0$ is chosen via cross validation. 
\item We stop the block coordinate descent solver for problem~\eqref{eq:tensor_regression_bound} when the update of factor matrices $\{{\mat{U}^{(k)}}\}_{k=1}^d$ is below a predefined threshold, or the algorithm reaches a predefined maximal number of iterations.
\end{itemize}

The overall algorithm, including an adaptive sampling which will be introduced in Section~\ref{sec:adap_sampling},  is summarized in Alg.~\ref{alg:ten_reg}.
\begin{algorithm}[t]
  \caption{Overall Adaptive Tensor Regression}
  \label{alg:ten_reg}
    \KwIn{Initial sample pairs $\left \{\vecpar_n, y(\vecpar_n)\right \}_{n=1}^N$, unitary polynomial order $p$, initial tensor rank $R$}
    \KwOut{Constructed surrogate model [Eq.~\reff{eq:builded_surrogate}]}
    \While{Adaptive sampling does not stop}{
    Construct the basis tensor $\ten{B}(\vecpar)$\\
    
    \eIf{No additional samples }
     {Initialize with Eq.~\reff{eq:ini_of_U}}
    {Initialize  $\{{\mat{U}^{(k)}}\}_{k=1}^d$ with the last solution}
    \While{Tensor regression does not stop}
    {
     \For{$k =1,2, \ldots ,d$}
        {
        update ${\mat{U}^{(k)}}$ via Eq.~\reff{eq:sol_sub_U}\\
        }
        Update $\vect{\eta}$ via 
        Eq.~\reff{eq:sol_sub_eta}\\
        Update regularization parameter $\lambda$ via Eq.~\reff{eq:update_lambda}
    }
    
    Shrink the tensor rank to $\hat{R}$ if possible\\
    Select new sample pairs based on Alg.~\ref{alg:adaptive_sampling}
    }
\end{algorithm}
After solving the factor matrices ${\{\mat{U}^{(k)}\}_{k=1}^d}$, \ie~the coefficient tensor $\ten{X}$, 
the surrogate on a sample $\vecpar$ can be efficiently calculated as
\begin{equation}\label{eq:builded_surrogate}
   \hat{y}(\vecpar) =  \langle \ten{X}, \ten{B}(\vecpar) \rangle= \sum_{r=1}^R \prod_{k=1}^d \left [{\boldsymbol{\phi}^{(k)}(\xi_k)}\right ]^T \mat{u}_{r}^{(k)}.
\end{equation}


In this work, tensor $\ten{X}$ is approximated by a low-rank CP decomposition. It is also possible to use other kinds of tensor decompositions. In those cases, although the tensor ranks are defined in different ways, the idea of enforcing group-sparsity over tensor factors still works. It is also worth noting that \eqref{eq:tensor_regression_bound} can be seen as a generalization of weighted group lasso. To further exploit the sparsity structure of the gPC coefficients, many variants can be developed from the statistic regression perspective, including the sparse group lasso, tensor-structured Elastic-Net regression, and so forth~\cite{hastie2015statistical}. 

\section{Adaptive Sampling Approach}
\label{sec:adap_sampling}
Another fundamental question in uncertainty quantification is how to select the parameter samples $\vecpar$ for simulation. 
We aim to reduce the simulation cost by selecting only a few informative samples for the detailed device- or circuit-level simulations.

Given a set of initial samples ${\Theta}$, we design a two-stage method to balance the exploration and exploitation in our active sampling process. In the first stage, we estimate the volume of some Voronoi cells via a Monte Carlo method to measure the sampling density in each region. In the second stage, we roughly measure the nonlinearity of $y(\vecpar)$ at some candidate samples via a Taylor expansion. We choose new samples that are located in a low-density region and make $y(\vecpar)$ highly nonlinear. In our implementation, the initial samples ${\Theta}  = \left \{\vecpar_n, y(\vecpar_n) \right \}_{n=1}^N $ are generated by the Latin Hypercube (LH) sampling method~\cite{mckay2000comparison}. Specifically, we first generate some standard LH samples $\{\boldsymbol{\zeta}^{\rm LH}_n\}_{n=1}^N$ in a hyper cube ${[0,1]}^d$,
then we transform them to the practical parameter space $ {\Omega}$ via the inverse transforms of the cumulative distribution function. 
Generally, the initial sample size of $\Theta$ depends on the number of unknowns in the model. Since problem~\eqref{eq:tensor_regression_bound} is regularized and solved via an alternating solver, given a limited simulation budget, we set the initial size $N$ to be smaller than the number of unknowns in our examples.

\subsection{Exploration: Volume Estimation of Voronoi Cells}
\label{sec:explore_sampling}
\begin{figure}[t]
    \centering
    \includegraphics[width=1\columnwidth]{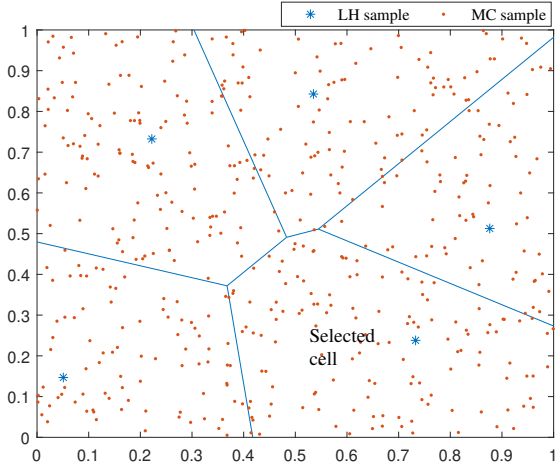}
    \caption{An example of Voronoi diagram on ${[0,1]}^2$. Each LH sample is a Voronoi cell center. The lower right cell should be selected in the first-stage since it has the largest estimated area (volume).}
    \label{fig:V_diagram}
\end{figure}   
Firstly, we employ an exploration step via a space-filling sequential design. Given the existing sample set $\Theta$, the sample density in $\Omega$ can be estimated via a Voronoi diagram~\cite{aurenhammer1991voronoi}. 
Specifically, each sample ${\vecpar_n}$ corresponds to a Voronoi cell $C_n \in {\Omega}$ that contains all the samples that lie closely to $\vecpar_n$ than other samples in ${\Omega}$. The Voronoi diagram is a complete set of cells that tesselate the whole sampling space. The volume of a cell reflects its sample density: a larger volume means that the cell region is less sampled.   

Here we provide a formal description of the Voronoi cell. Given two distinct samples $\vecpar_i, \vecpar_j \in \Omega$, there always exist a half-plane $\hp(\vecpar_i,\vecpar_j)$ that contains all samples that are at least as close to $\vecpar_i$ as to $\vecpar_j$
\begin{equation}
    \hp(\vecpar_i,\vecpar_j)=\{\vecpar \in \R^d | \; \|\vecpar-\vecpar_i\| \le \|\vecpar-\vecpar_j\|\}.
\end{equation}
The Voronoi cell $C_i$ is defined as the space that lie in the intersection of all half-plane $\hp(\vecpar_i,\vecpar_j), \forall \vecpar_j\in \Omega \setminus \vecpar_i$: 
\begin{equation}
         C_i = \bigcap_{\vecpar_j\in \Omega \setminus \vecpar_i} \hp(\vecpar_i,\vecpar_j).
\end{equation}  
It is intractable to construct a precise Voronoi diagram and calculate the volume exactly in a high-dimensional space. Fortunately, we do not need to construct the exact Voronoi diagram. Instead, we only need to estimate the volume in order to measure the sample density in that cell. This can be done via a Monte Carlo method.
\begin{observation}\label{obs:distance}
In order to detect the least-density region in $\Omega$, we can either estimate the density of $\Omega$ directly or estimate the density of hyper-cube  ${[0,1]}^d$ and then transform it to $\Omega$.
In the Monte-Carlo-based density estimation, it is fairer to choose the latter one.

\end{observation}
\begin{example}
One simple example is given in Appendix~\ref{sec:distance_example}. 
\end{example}

Based on the above observation, we estimate the volume of Voronoi cell in the hyper cube.
Let the existing LH samples $\{\boldsymbol{\zeta}^{\rm LH}_n\}_{n=1}^N$ be the cell centers $\{C_n\}_{n=1}^N$.
We first randomly generate $M$ Monte Carlo samples $\{ \boldsymbol{\psi}_m \}_{m=1}^M \in  {[0,1]}^d$. 
For each random sample, we calculate its Euclidean distance towards the cell centers and assign it to the closest one.
Then the volume of the cell ${{\rm vol}(C_n)}$ 
is estimated by counting the number of assigned random samples. 
The cell with the largest estimated volume is least-sampled. The MC samples assigned to this cell are denoted as set $\Gamma$. A simple example that illustrates the first-round search is shown in Fig.~\ref{fig:V_diagram}. 
After transforming all Monte Carlo samples in set $\Gamma$ 
back to the actual parameter space $\Omega$ via the inverse transform sampling method, we obtain a set of candidates for the next-stage selection, denoted as set $\Gamma_{\Omega}$. 

The accuracy of volume estimation depends on the number of random samples. Clearly, more Monte Carlo samples can estimate the volume more accurately, but they also induce more computational burden. 
As suggested by~\cite{crombecq2009space}, to achieve a good estimation accuracy, we use $M$ Monte Carlo samples with $M=100N$.

\subsection{Exploitation: Nonlinearity Measurement}
\label{sec:nonlinear_sampling}

In the second stage, we aim to do an exploitation search based on the obtained candidate sample set $\Gamma_{\Omega}$.
Based on the assumption that the region with a more nonlinear response is harder to capture, we choose the criterion in the second stage as the nonlinearity measure of the target function.
We know that the first-order Taylor expansion of a function becomes more inaccurate if that function is more nonlinear. Therefore, given a sample $\vecpar$, we measure the non-linearity of $y(\vecpar)$ via the difference of $y(\vecpar)$ and its first-order Taylor expansion around the closest Voronoi cell center $\mat{a} \in \Omega$~\cite{mo2017taylor}. We do not know exactly the expression of $y(\vecpar)$, but we have already built a surrogate model $\hat{y}(\vecpar)$ based on previous simulation samples. Therefore, the nonlinearity of $y(\vecpar)$ can be roughly estimate as
\begin{equation}\label{eq:nonlinear_measure}
    \gamma(\vecpar) = |\hat{y}(\vecpar) - \hat{y}(\mat{a}) - \nabla \hat{y}(\mat{a})^T(\vecpar-\mat{a}) |.
\end{equation}
Notice that the nonlinear measure does not imply the accuracy of the surrogate model since we do not use the simulation value here.
In the second stage, we will choose the sample $\vecpar^{\star}$ that has the largest $\gamma(\vecpar)$ from the candidate set of $\Gamma_{\Omega}$: 
\begin{equation}
\label{eq:max_nonlinear}
 \vecpar^{\star}=\argmax \limits_{\vecpar \in \Gamma_{\Omega}}  \left(\gamma \left( \vecpar \right) \right). 
\end{equation}

To summarize, we select
the most nonlinear sample from the least-sampled cell space, which is a good trade-off between exploration and exploitation. 
Based on the above, we summarize the adaptive sampling procedure in Alg.~\ref{alg:adaptive_sampling}.
\begin{algorithm}[t]
  \caption{Adaptive sampling procedure}
  \label{alg:adaptive_sampling}
  \KwIn{Initial samples pairs $\Theta = \left \{\vecpar_n, y(\vecpar_n)\right \}_{n=1}^N$}
  \KwOut{Sample pairs $\Theta^{\star}$ with the additional sample}
    Uniformly generate $M=100N$ Monte Carlo samples $\{\boldsymbol{\psi}_{m}\}_{m=1}^M \in {[0,1]}^d$\\
     \For{$m =1,2, \ldots,M$}{
     Find the closest cell $C_n$ center to ${\boldsymbol{\psi}_m}$\\
     ${\rm vol}(C_n) \leftarrow {\rm vol}(C_n) + 1$\\ 
     }
     Find the cell with the biggest estimated volume ${\rm vol}$ and the sample set $\Gamma$ assigned to this cell\\
     $\Gamma_{\Omega} \leftarrow  \text{Inverse\_transform\_sampling}(\Gamma)$ \\
     Calculate the nonlinearity measure $\gamma(\Gamma_{\Omega})$ via Eq.~\reff{eq:nonlinear_measure}\\
    
    Select $\vecpar^{\star} $ according to Eq.~\eqref{eq:max_nonlinear}\\
     
     $\Theta^{\star} \leftarrow \Theta \bigcup \{ \vecpar^{\star}, y(\vecpar^{\star})\}$
\end{algorithm}

\subsection{Discussion}
The proposed adaptive sampling method can be easily extended to a batch version by searching for the top-$K$ least-sampled regions in the first stage. We can stop sampling when we exceed a sampling budget or when the constructed surrogate model achieves the desired accuracy.

The sampling criteria do not rely on the structure of the targeted surrogate model. 
Therefore, the proposed sampling method is very flexible and generic.
The proposed method is very suitable for constructing a high-dimensional polynomial model due to two reasons.
Firstly, the number of samples required in estimating the Voronoi cell does not rely on the parameter dimensionality but on the number of existing samples. 
Secondly, the derivative and the nonlinearity of the surrogate model are easy to compute.

Some variants of the proposed sampling methods may be further developed.
For instance, we may define a score function as the combination of the estimated volume 
and the nonlinearity measure, and then calculate the score for each Monte Carlo sample and select the best one. In the batch version, we may also select several top nonlinear samples from the same Voronoi cell.




\section{Statistical Information Extraction}
\label{sec:post-processing}
Based on the obtained tensor regression model $\hat{y}(\vecpar) = \sum \limits_{\basisInd \in \Theta} {c}_{\basisInd}\multiGPC_{\basisInd}(\vecpar)=\langle \ten{X}, \ten{B}(\vecpar) \rangle$, we can easily extract important statistical information such as moments and Sobol' indices.

\begin{figure*}[t]
    \centering
    \includegraphics[width=2\columnwidth]{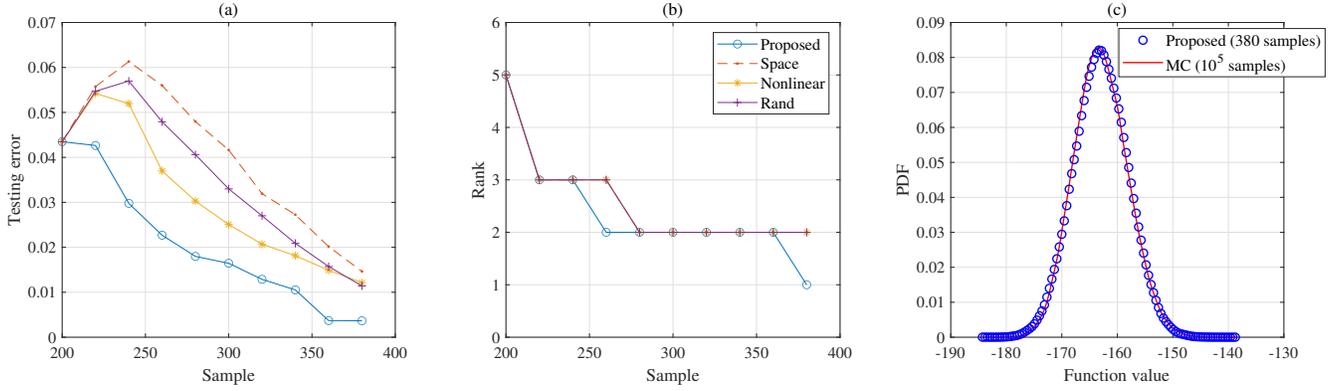}
    \caption{Results of approximating the synthetic function. (a) Testing error on $10^5$ MC samples. (b) The estimated rank. (c) Probability density functions of the function value. }
    \label{fig:syn_fun}
\end{figure*}
\begin{figure}[t]
    \centering
    \includegraphics[width=1\columnwidth]{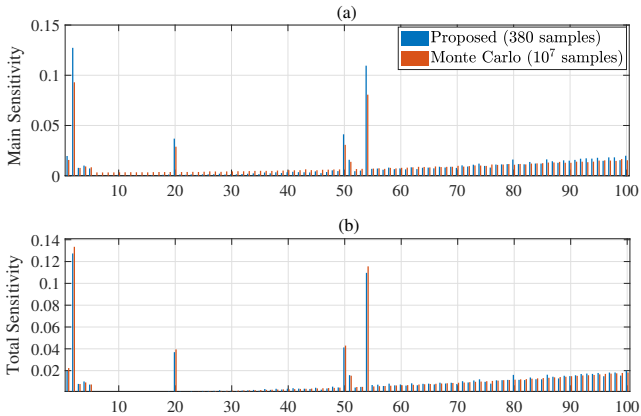}
    \caption{Sensitivity analysis of the synthetic function in \reff{eq:syn_function}. The proposed method fits the results from Monte Carlo~\cite{saltelli2010variance} with $10^7$ simulations very well.}
    \label{fig:sen_syn}
\end{figure}

\begin{itemize}[leftmargin=*]
    \item 
    \textbf{Moment information.} 
The mean $\mu$ of the constructed $\hat{y}(\cdot)$ is the coefficient of the zero-order basis $\multiGPC_{\mat{0}}(\vecpar)$:
\begin{equation}\label{eq:post_mean}
    \mu = c_{\mat{0}}=x_{1 1 \cdots 1} = \sum_{r=1}^R \vect{u}_r^{(1)}(1) \vect{u}_r^{(2)}(1) \cdots \vect{u}_r^{(d)}(1).
\end{equation}
$x_{1 1 \cdots 1}$ is the $(1,1,\cdots,1)$-th element of tensor $\ten{X}$, and $\mat{u}_r^{(k)} (1)$ denotes the first element of vector $\mat{u}_r^{(k)}$
The variance of $y(\vecpar)$ can be estimated as:
\begin{equation}\label{eq:post_var}
\begin{aligned}
    \sigma^2 = &  \sum \limits_{\basisInd \in \Theta, \basisInd \neq 0} {c}_{\basisInd} =\langle \ten{X},\ten{X} \rangle - x^2_{1 1 \cdots 1}     \\
    = & \sum_{r_1=1}^R \sum_{r_2=1}^R \prod_{k=1}^d {\mat{u}^{(k)}_{r_1}}^{T} {\mat{u}^{(k)}_{r_2}} - \mu^2.
\end{aligned}
\end{equation}
    \item \textbf{Sobol' indices.} 
Based on the obtained model, we can also extract the Sobol' indices~\cite{sobol2001global,saltelli2010variance} for global sensitivity analysis. 
The main sensitivity index $S_j$ measures the contribution by random parameter $\xi_j$ along to the variance $y(\vecpar)$:
\begin{equation}\label{eq:sensitive_main}
\begin{aligned}
&S_j = \frac{\text{Var}\left[ \E \left[ y(\vecpar) | \xi_j\right]  \right]}{\sigma^2} 
\end{aligned}
\end{equation}
where $\E \left[ y(\vecpar) | \xi_j\right]$ denotes the conditional expectation of $y(\vecpar)$ over all random variables except $\xi_j$. The variance of this conditional expectation can be estimated as
\begin{equation}
\begin{aligned}
     \text{Var}\left[ \E \left[ y(\vecpar) | \xi_j\right]  \right] = & \sum_{i_j=2}^{p+1}x_{1\cdots i_{j} 1\cdots1}^2\\
    =&
   \sum_{i_j=2}^{p+1}{\left [\sum_{r=1}^R  \vect{u}_r^{(j)}(i_j) \prod \limits_{k\neq j}  \vect{u}_r^{(k)}(1) \right ]}^2.
\end{aligned}
\end{equation}
The total sensitivity index $T_j$ measures the contribution to the variance of $y(\vecpar)$ by variable $\vecpar_j$ and its interactions with all other variables:
\begin{equation}\label{eq:sensitive_total}
\begin{aligned}
& T_j = 1- \frac{\text{Var}\left[ \E \left[ y(\vecpar) | \vecpar_{\backslash j} \right]  \right]}{\sigma^2}.
\end{aligned}    
\end{equation}
Here $\vecpar_{\backslash j}$ includes all elements of $\vecpar$ except $\xi_j$. The involved variance of a conditional expectation is estimated as

\begin{equation}
\begin{aligned}
    &\text{Var}\left[ \E \left[ y(\vecpar) | \vecpar_{\backslash j}\right]  \right] \\
    =&\sum_{(i_1,i_2, \cdots i_d), \; i_j=1} x^2_{i_1 \cdots i_d} - x^2_{1 1 \cdots 1} \\
    =&\sum_{r_1=1}^R \sum_{r_2=1}^R {\mat{u}^{(j)}_{r_1}(1)} {\mat{u}^{(j)}_{r_2}(1)} \prod_{ k\not=j} {\mat{u}^{(k)}_{r_1}}^{T} {\mat{u}^{(k)}_{r_2}}  - \mu^2.
\end{aligned}    
\end{equation}
Similarly, we can also express any higher-order index representing the effect from the interaction between a set of variables with an analytical form.
\end{itemize}

\section{Numerical Results}
\label{sec:numerical_results}
In this section, we will verify the proposed tensor-regression uncertainty quantification method in one synthetic function and two photonic/ electronic IC benchmarks. 

\subsection{Baseline Methods for Comparison} 
We compare our proposed method with the following approaches. 
\begin{itemize} 
    \item Tensor regression with adaptive sampling based on space exploration only introduced in Section~\ref{sec:explore_sampling} (denoted as Space). 
    \item Tensor regression with adaptive sampling based on exploiting nonlinearity only introduced in model with only Section~\ref{sec:nonlinear_sampling} (denoted as Nonlinear).
    \item Tensor regression model with random sampling  (denoted as Rand). In each iteration of adding samples, new samples are simply randomly selected. 
    \item Fixed-rank tensor regression (denoted as Fixed rank). This method uses a tensor ridge regularization in the regression objective function~\cite{guo2011tensor}:
    \begin{equation}\label{eq:tensor_ridge_regression}
    \begin{aligned}
            \min_{\{\mat{U}^{(k)}\}_{k=1}^d} f(\ten{X}) = &  h(\ten{X})
            +  \lambda \sum_{k=1}^d \|\mat{U}^{(k)}\|_{\rm F}^2.
    \end{aligned}
    \end{equation}
    The standard ridge regression does not induce a sparse structure.
    We will keep the tensor rank fixed in solving Eq.~\reff{eq:tensor_ridge_regression}.
    
    \item Sparse gPC expansion with a total degree truncation~\cite{luthen2020sparse} (denoted as Sparse gPC). With the truncation scheme in Eq.~\reff{eq:total_scheme}, we compute the gPC coefficients by solving the following problem:
    \begin{equation}
    \label{eq:classical_lasso}
        \min_{\vect{\hat{c}}}  \frac{1}{2} \sum_{n = 1}{\left( y_n -  \sum_{\basisInd \in \Theta} \hat{c}_{\basisInd}\multiGPC_{\basisInd}(\vecpar_{n}) \right)}^2 + \lambda \|\vect{\hat{c}}\|_1.
    \end{equation}
\end{itemize}
\begin{table}[t]
\centering
\caption{Model Comparisons on the Synthetic Function}
\label{tab:syn_function}
\begin{threeparttable}
\begin{tabular}{ccccccc}
\toprule
& Sample \#     & Variable \#  & Mean &  Std & Testing\\
\midrule
Monte Carlo  & $10^5$ & N/A   & -162.95       &  4.80    &  N/A \\
Sparse gPC   & 380  & 5151      & -163.04         &  2.27  & 2.3\%\\
Fixed rank   & 380  & 15x100      & -163.24        &  5.02  & 1.01\%\\
\textbf{Proposed}   & 380  & 3x100\tnote{*}  & -162.93    &  4.86 & 0.37\%\\
\bottomrule
\end{tabular}
\begin{tablenotes}
   \item[*] In the alternating solver, there are 100 subproblems with 3 unknown variables in each one (the rank has been shrunk). 
  \end{tablenotes}
\end{threeparttable}
\end{table}

\subsection{Synthetic Function (100-dim)}
We first consider the following high-dimensional analytical function~\cite{marelli2014uqlab}:
\begin{align}\label{eq:syn_function}
     y(\vecpar) & =  3 - \frac{5}{d}\sum \limits_{k=1}^d k \xi_k + \frac{1}{d} \sum \limits_{k=1}^d k \xi_k^3  +  \xi_1 \xi_2^2 + \xi_2 \xi_4  \nonumber \\
    & - \xi_3 \xi_5 + \xi_{51} + \xi_{50} \xi_{54}^2+ \ln{(\frac{1}{3d}\sum \limits_{k=1}^d k(\xi_k^2 + \xi_k^4))}
\end{align}
where dimension $d=100$, $\xi_{20} \sim \U([1,3])$, and $\xi_k \sim \U([1,2]), k\not=20$. 
We aim to approximate $f(\vecpar)$ by a tensor-regression gPC model and perform sensitivity analysis. 

Assume that we use 2nd-order univariate basis functions for each random variable, then we will need $3^{100}$ multi-variate basis functions in total. To approximate the coefficient tensor, we initialize it with a rank-5 CP decomposition and use $q=0.5$ in regularization. We initialize the training with 200 Latin-Hypercube samples and adaptively select 9 batches of additional samples, with each batch having 20 new samples. We test the accuracy of different models on additional $10^5$ samples. 
Fig.~\ref{fig:syn_fun}~(a) shows the relative $\ell_2$ testing errors (\ie, $ \frac{\|\mat{y}(\vecpar) - \hat{\mat{y}}(\vecpar) \|_2}{\|\mat{y}(\vecpar)\|_2}$) of different methods. 
The testing errors may not monotonically decrease since more samples can not strictly guarantee the convergence of the surrogate model. However, see from the figure, we can generally conclude that more training samples lead to a better model approximation and the proposed sampling method outperforms the others.
Fig.~\ref{fig:syn_fun}~(b) shows the estimated tensor rank as the number of training samples increases.
The proposed method shrinks the tensor rank differently from other methods while achieving the best performance.   
It shows that a correct determination of the tensor rank helps the function approximation.  
Fig.~\ref{fig:syn_fun}~(c) plots the predicted probability density function of our obtained model which is estimated via a kernel density estimator. It matches the Monte Carlo simulation result of the original function very well. 

We compare the complexity and accuracy of different methods in Table~\ref{tab:syn_function}. We treat the result from $10^5$ Monte Carlo simulations as the ground truth. 
For the other models, the mean and standard deviation are both extracted from the polynomial coefficients. 
Given the same amount of (limited) training samples, the proposed method achieves the highest approximation accuracy. 


Now we perform sensitivity analysis to identify the random variables that are most influential to the output. Fig.~\ref{fig:sen_syn} plots the main and total sensitivity metrics from the proposed method and a Monte Carlo estimation~\cite{saltelli2010variance} with $10^7$ simulations. With much fewer function evaluations, our proposed method can precisely identify the indices of some most dominant random variables that contribute to the output variance.

\begin{figure}[t]
    \centering
    \includegraphics[width=2.8in]{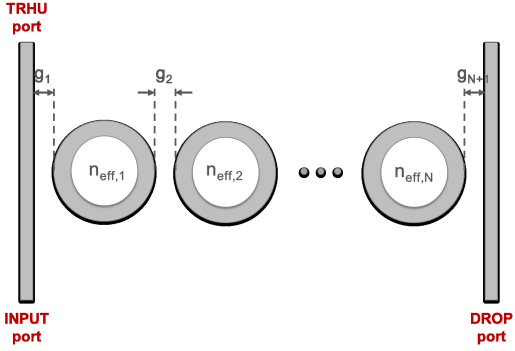}
    \caption{Schematic of a band-pass filter with 9 micro-ring resonators.}
    \label{fig:PIC_image}
\end{figure}

\begin{table}[t]
\centering
\caption{Model Comparisons on the Photonic Band-pass Filter}
\label{tab:pic19}
\begin{tabular}{ccccccc}
\toprule
& Sample \#     & Variable \#  & Mean &  Std & Error\\
\midrule
Monte Carlo  & $10^5$ & N/A   & 21.6511       &  0.0988 &  N/A \\
Sparse gPC   & 100  & 210     & 21.6537       &  0.0735 & 0.39\%\\
Fixed rank   & 100  & 12x19   & 21.6677       &  0.1906 & 0.52\%\\
\textbf{Proposed}   & 100  & 3x19  & 21.6567 &  0.0955  & 0.16\%\\
\bottomrule
\end{tabular}
\end{table}

\begin{figure*}[t]
    \centering
    \includegraphics[width=2\columnwidth]{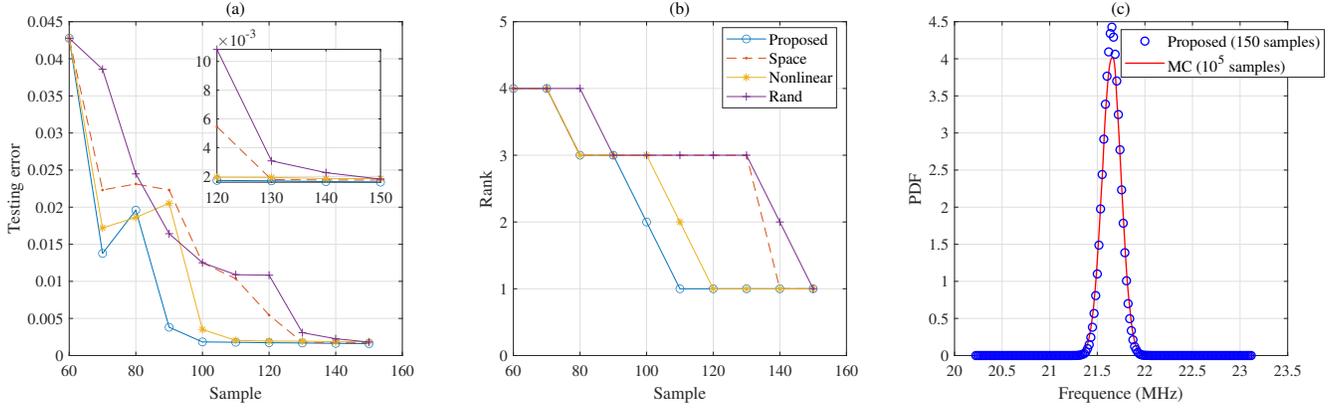}
    \caption{Result of the photonic filter. (a) Testing error on $10^5$ MC samples. (b) The estimated rank. (c) Probability density functions of the 3-dB bandwidth $f_{\text{3dB}}$ at the DROP port.}
    \label{fig:pic19}
\end{figure*}
\begin{figure}[t]
    \centering
    \includegraphics[width=1\columnwidth]{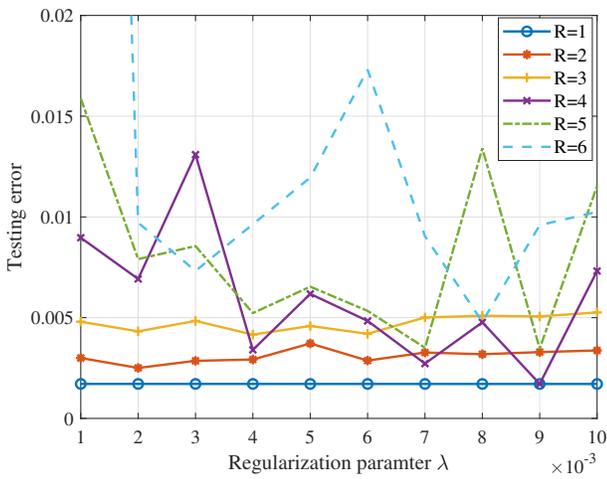}
    \caption{One-shot approximations for the photonic band-pass filter with 800 training samples under different ranks and $\lambda$. The rank-1 initialization works the best in this example.}
    \label{fig:pic19_lr}
\end{figure}

\subsection{Photonic Band-pass Filter (19-dim)}

Now we consider the photonic band-pass filter in Fig.~\ref{fig:PIC_image}. This photonic IC has 9 micro-ring resonators, and it was originally designed to have a 3-dB bandwidth of 20 GHz, a 400-GHz free spectral range, and a 1.55-nm operation wavelength. A total of 19 independent Gaussian random parameters are used to
describe the variations of the effective phase index ($n_{\text{neff}}$) of each ring, as well as the gap (g) between adjacent rings and between the first/last ring and the bus waveguides. We aim to approximate the 3-dB bandwidth $f_{\text{3dB}}$ at the DROP port as a tensor-regression gPC model.

We use 2nd order univariate polynomial basis functions for each random parameter and have $3^{19}$ multivariate basis functions in total in the tensor regression gPC model. We initialize the gPC coefficients as a rank-4 CP tensor decomposition and set $q=0.5$ in our regularization. We initialize the training with 60 Latin-Hypercube samples and adaptively select 9 batches of additional samples, with each batch have 10 new samples. We test the obtained model with additional $10^5$ samples. 
Fig.~\ref{fig:pic19}~(a) shows the relative $\ell_2$ testing errors. The proposed method outperforms the others in the first few adaptive sampling rounds. All the models perform similarly when the ranks are all shrunk to 1.  
Fig.~\ref{fig:pic19}~(b) shows the estimated tensor rank as the number of training samples increases. The tensor ranks are shrunk gradually in all cases, but our proposed method finds the best rank with minimal samples. 
Fig.~\ref{fig:pic19}~(c) plots the predicted probability density function of our obtained result. Since the benchmark has a relatively small standard deviation, the limited approximation error is revealed as the discrepancy around the peak.

In order to see the influence of the tensor rank initialization, we do the one-shot approximation with different initial tensor ranks $R$ and different regularization parameters $\lambda$ as illustrated in Fig.~\ref{fig:pic19_lr}. 
For a specific benchmark, the best-estimated tensor rank highly depends on the number of training samples. 
Given the limited number of simulation samples, the rank-1 initialization works the best in this example. It coincides with the results shown in Fig.~\ref{fig:pic19}, where the predicted rank is $1$.
We also compare the complexity and accuracy of all methods in Table~\ref{tab:pic19}. 
The proposed method achieves the best accuracy with limited simulation samples.

\subsection{CMOS Ring Oscillator (57-dim)}
\begin{figure}[t]
    \centering
    \includegraphics[width=0.75\columnwidth]{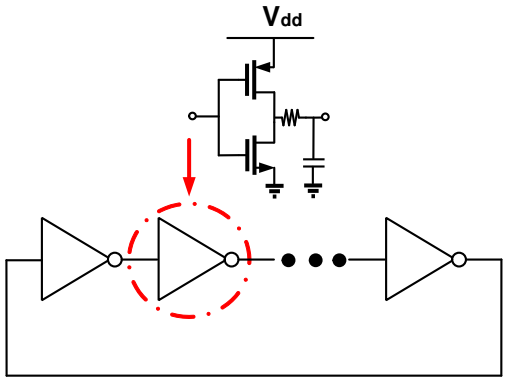}
    \caption{Schematic of a CMOS ring oscillator.}
    \label{fig:CMOS_image}
\end{figure}
\begin{figure}[t]
    \centering
    \includegraphics[width=1\columnwidth]{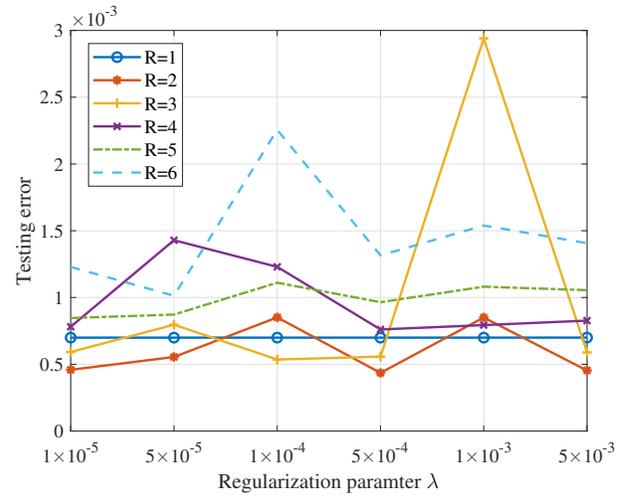}
    \caption{One-shot approximations for the CMOS ring oscillator with 150 training samples under different ranks and $\lambda$. In this example, the rank-2 model works the best in most cases.}
    \label{fig:cmos57_lr}
\end{figure}

\begin{figure*}[t]
    \centering
    \includegraphics[width=2\columnwidth]{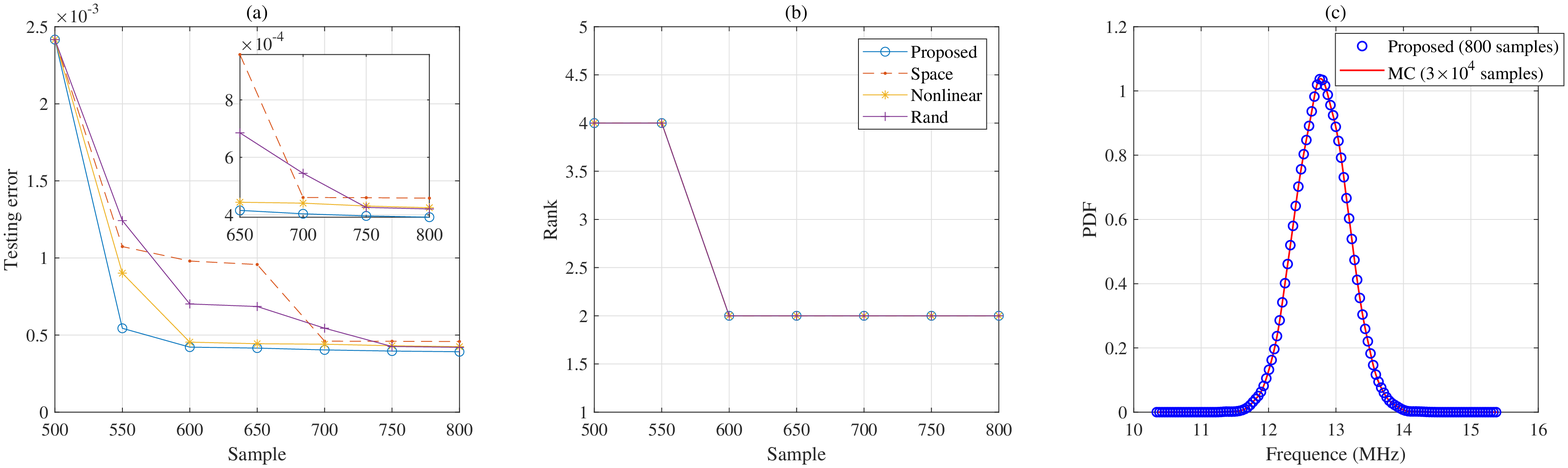}
    \caption{Results of the CMOS ring oscillator. (a) Testing error on $3\times10^4$ MC samples. (b) The estimated rank. (c) Probability density functions of the oscillator frequency.}
    \label{fig:cmos57}
\end{figure*}

We continue to consider the 7-stage CMOS ring oscillator in Fig.~\ref{fig:CMOS_image}. This circuit has 57 random variation parameters, including Gaussian parameters describing the temperature, variations of threshold voltages and gate-oxide thickness, and uniform-distribution parameters describing the effective gate length/width. We aim to approximate the oscillator frequency with tensor-regression gPC under the process variations. 

We use 2nd-order univariate basis functions for each random parameter, leading to $3^{57}$ multivariate basis functions in total. We initialize the gPC coefficients as a rank-4 tensor and set $q=0.5$ in the regularization term. We initialize the training with $500$ Latin-Hypercube samples and adaptively select 300 additional samples in total by 6 batches. We test the obtained model with $3\times 10^4$ additional samples. 
Fig.~\ref{fig:cmos57}~(a) shows the relative $\ell_2$ testing errors of all methods. 
The proposed method outperforms other methods significantly when the number of samples is small. Fig.~\ref{fig:cmos57}~(b) shows that the estimated tensor rank reduces to $2$ in all methods. 
Fig.~\ref{fig:cmos57}~(c) plots the predicted probability density function of the obtained tensor regression model, which is indistinguishable from the result of Monte Carlo simulations. 

\begin{table}[t]
\centering
\caption{Model Comparisons on the CMOS Ring Oscillator}
\label{tab:ring57}
\begin{tabular}{ccccccc}
\toprule
& Sample \#     & Variable \#  & Mean &  Std & Error\\
\midrule
Monte Carlo  & $3\times10^4$ & N/A   & 12.7920      &  0.3829    &  N/A \\
Sparse gPC   & 600  & 1711      & 12.7931         &  0.3777  & 0.11\%\\
Fixed rank   & 600  & 12x57      & 12.7929         &  0.3822  & 0.10\%\\
\textbf{Proposed} & 600  & 6x57  & 12.7918    &  0.3830 & 0.04\%\\
\bottomrule
\end{tabular}
\end{table}

We do the one-shot approximation with different initial tensor ranks $R$ and different regularization parameters $\lambda$ as illustrated in Fig.~\ref{fig:cmos57_lr}. Conforming with the results shown in Fig.~\ref{fig:cmos57}, a rank-2 model is more suitable in this example. 
We compare the proposed method with the fixed rank model and the 2nd-order sparse gPC in Table~\ref{tab:ring57}, where the proposed compact tensor model is shown to have the best approximation accuracy.

\section{Conclusion}
\label{sec:conclusion}
This paper has proposed a tensor regression framework for quantifying the impact of high-dimensional process variations. By low-rank tensor representation, this formulation can reduce the number of unknown variables from an exponential function of parameter dimensionality to only a linear one. Therefore it works well with a limited simulation budget. We have addressed two fundamental challenges: automatic tensor rank determination and adaptive sampling. The tensor rank is estimated via a $\ell_{q}/ \ell_{2}$-norm regularization. 
The simulation samples are chosen based on a two-stage adaptive sampling method, which utilizes the Voronoi cell volume estimation and the nonlinearity measure of the quantity of interest. 
Our model has been verified by both synthetic and realistic examples with $19$ to $100$ random parameters. The numerical experiments have shown that our method can well capture the high-dimensional stochastic performance with much fewer simulation data.  



\appendices

\section{Proof of Lemma~\ref{lemma:variational}}
\label{sec:variational_ineq}

We consider two cases $\alpha \in \left(0,2\right)$~\cite{jenatton2010structured} and $\alpha = 2$~\cite{guo2011tensor}.

When $\alpha \in \left(0,2\right)$, $\kappa(\boldsymbol{\eta}) :=  \frac{1}{2}\sum_{r=1}^p \frac{y^2_r}{\eta_r} + \frac{1}{2}\|\boldsymbol{\eta}\|_\beta$ is a continuously differentiable function for any $\eta_i \in \left(0, \infty \right)$. When $y_r \not= 0$, $\lim \limits_{\eta_r \rightarrow \infty} \kappa(\boldsymbol{\eta}) = \infty$ and $\lim \limits_{\eta_r \rightarrow 0} \kappa(\boldsymbol{\eta}) = \infty$.
Therefore, the infimum of $\kappa(\boldsymbol{\eta})$ exists and it is attained.
According to the first-order optimality and enforcing the derivative w.r.t. $\eta_r~(\eta_r>0)$ to be zero, we can obtain
\begin{equation}
\eta_r = 	{\left|y_r\right|}^{2-\alpha}  \|\boldsymbol{\eta}\|_{\frac{\alpha}{2-\alpha}}^{\alpha-1} .
\end{equation}
With $y_r = \|\boldsymbol{\eta}\|_{\frac{\alpha}{2-\alpha}}^{\frac{1-\alpha}{2-\alpha}} {\eta_r}^{\frac{1}{2-\alpha}}$, 
we have
$\|\boldsymbol{\eta}\|_{\frac{\alpha}{2-\alpha}} = \|\boldsymbol{\eta}\|_{\frac{\alpha}{2-\alpha}}^{\frac{1-\alpha}{2-\alpha}} {(\sum_{r=1}^p {\eta_r}^{{\frac{\alpha}{2-\alpha}}})}^{{\frac{1}{\alpha}}} = \|\mat{y}\|_{\alpha},$
therefore we obtain the optimal solution $\eta_r = 	{\left|y_r\right|}^{2-\alpha}  \|\mat{y}\|_{\alpha}^{\alpha-1} $ in Lemma~\ref{lemma:variational}.  
If $y_r = 0$, the solution to $\min_{\boldsymbol{\eta}\ge 0} \kappa(\boldsymbol{\eta})$ is $\eta_r=0$, which is also consistent with Lemma~\ref{lemma:variational}.

When $\alpha = 2$, $\|\boldsymbol{\eta}\|_1$ is non-differentiable. Given a scalar $y_r$, we have $y_r = \frac{{y_r}^2}{2\eta_r}+\frac{1}{2}\eta_r$ only when $\eta_r=y_r$  (we let $\frac{{y_r}^2}{2\eta_r} = 0$ when $y_r=\eta_r=0$). 
Similarly, given a vector $\mat{y} \in \R^p$, 
we have $\|\mat{y}\|_1 = \frac{1}{2}\sum_{r=1}^p \frac{{y_r}^2}{\eta_r} + \frac{1}{2}\|\boldsymbol{\eta}\|_1$ only when $\boldsymbol{\eta}=\left|\mat{y}\right|$, which is also consistent with Lemma~\ref{lemma:variational}.

\section{Detailed Derivation of Eq.~\reff{eq:sol_sub_U}}
\label{sec:BCD_solution}
Let $\mat{\Lambda} = \text{diag}(\frac{1}{{\eta}_1},\ldots,\frac{1}{{\eta}_{R}})$ and $*$ denote a Hadamard product, 
we can rewrite the objective function of an $\mat{U}^{(k)}$-subproblem as
\begin{equation}
\begin{aligned}
    &f_k(\mat{U}^{(k)})   \\
    = &\frac{1}{2} \sum_{n=1}^N \left[ y_n -
    \langle \mat{U}^{(k)} {\mat{U}^{(\setminus k)}}^T, \mat{B}^n_{(k)}  \rangle 
    \right]^2 + \frac{\lambda}{2}\sum_{r=1}^R \frac{ \| \mat{u}_{r}^{(k)} \|_2^2 }{{\eta}_r}\\
    = & \frac{1}{2} \sum_{n=1}^N \left[ y_n -
    \Tr\left(\mat{U}^{(k)} {\left( \tilde{\mat{B}}_{(k)}^n \right) } ^T\right) 
    \right]^2 + \frac{\lambda}{2}\Tr(\mat{U}^{(k)} \mat{\Lambda} {\mat{U}^{(k)}}^T) \nonumber
\end{aligned}
\end{equation}
with $\tilde{\mat{B}}^n_{(k)} = \mat{B}^n_{(k)} \mat{U}^{(\setminus k)}$. 
When the dimension $d$ is large,  it is intractable to store and compute $\mat{B}^n_{(k)} \in \R^{(p+1) \times (p+1)^{(d-1)}}$ or $\mat{U}^{(\setminus k)} \in \R^{(p+1)^{(d-1)} \times R}$. Fortunately, based on the property of Khatri-Rao product, we can compute $\tilde{\mat{B}}^n_{(k)}$ as 
\begin{equation*}
\begin{aligned}
&\tilde{\mat{B}}^n_{(k)} = \mat{B}^n_{(k)} \mat{U}^{(\setminus k)} \\
&= \vect{\uniGPC}^{(k)}(\xi_k^n)[\vect{\uniGPC}^{(d)}(\xi_d^n)^T \mat{U}^{(d)} * \cdots  \vect{\uniGPC}^{(k+1)}(\xi_{k+1}^n)^T \mat{U}^{(k+1)}  *\\ 
&\vect{\uniGPC}^{(k-1)}(\xi_{k-1}^n)^T \mat{U}^{(k-1)} * 
\cdots \vect{\uniGPC}^{(1)}(\xi_1^n)^T \mat{U}^{(1)}].
\end{aligned}
\end{equation*}
Enforcing the following 1st-order optimality condition 
\begin{equation}
\begin{aligned}
&\frac{\partial f_k(\mat{U}^{(k)})}{\partial \mat{U}^{(k)}} = \\
&-\frac{1}{2} \sum_{n=1}^N \left[ y_n -\Tr( \mat{U}^{(k)} {(\tilde{\mat{B}}_{(k)}^n)} ^T) \right] \tilde{\mat{B}}_{(k)}^n + \lambda \mat{U}^{(k)}\mat{\Lambda} = \mat{0}, \nonumber
\end{aligned}
\end{equation}
we can obtain the analytical solution in Eq.~\reff{eq:sol_sub_U}.

\section{An example to show Observation~\ref{obs:distance}}
\label{sec:distance_example}
Suppose we already have two samples ${[0.2, 0.6]}$, and we consider a candidate sample ${0.4}$ in the interval $[0,1]$ equipped with a uniform distribution. Then, based on Box–Muller transform, their corresponding Gaussian-distributed samples are ${[-0.8416,0.2533]}$ and $-0.2533$, respectively.
It is easy to know that the PDF value of sample $0.2533$ is larger than sample $-0.8416$ in a standard Gaussian distribution.
Apparently, the candidate sample is equally close to the two examples in a uniform-sampled space, but it is closer to the one with a higher probability density in the Gaussian-sampled space. 


\small{
\bibliographystyle{IEEEtran}
\bibliography{Bib/reference}
}

\begin{IEEEbiography}
[{\includegraphics[width=1in,height=1.25in,clip,keepaspectratio]{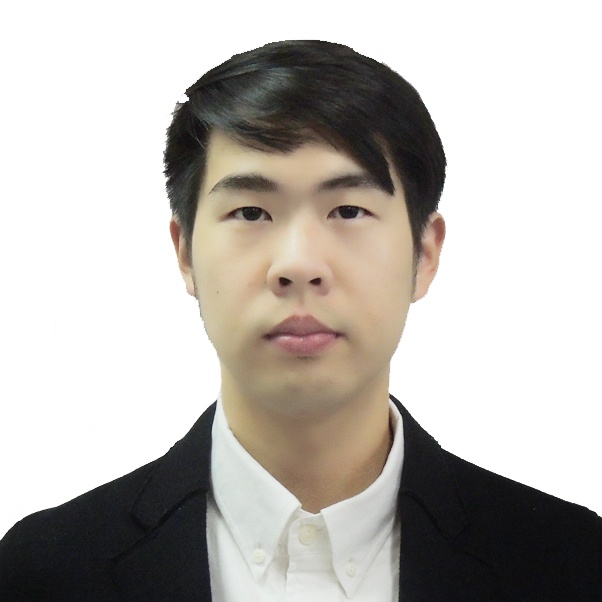}}]{Zichang He} received the B.E. degree in Detection, Guidance and Control Technology in 2018 from Northwestern Polytechnical University, Xi'an, China. In 2018 he joined the Department of Electrical and Computer Engineering at University of California, Santa Barbara as a Ph.D. student.

Zichang's research activities are mainly focused on uncertainty quantification and tensor related topics with applications on design automation, machine learning, and quantum computing. He is the recipient of best student paper award in IEEE Electrical Performance of Electronic Packaging and Systems (EPEPS) conference in 2020 and the Outstanding Teaching Assistant award in the department of ECE at UCSB in 2020 and 2021.
\end{IEEEbiography}

\begin{IEEEbiography}
 [{\includegraphics[width=1in,height=1.25in,clip,keepaspectratio]{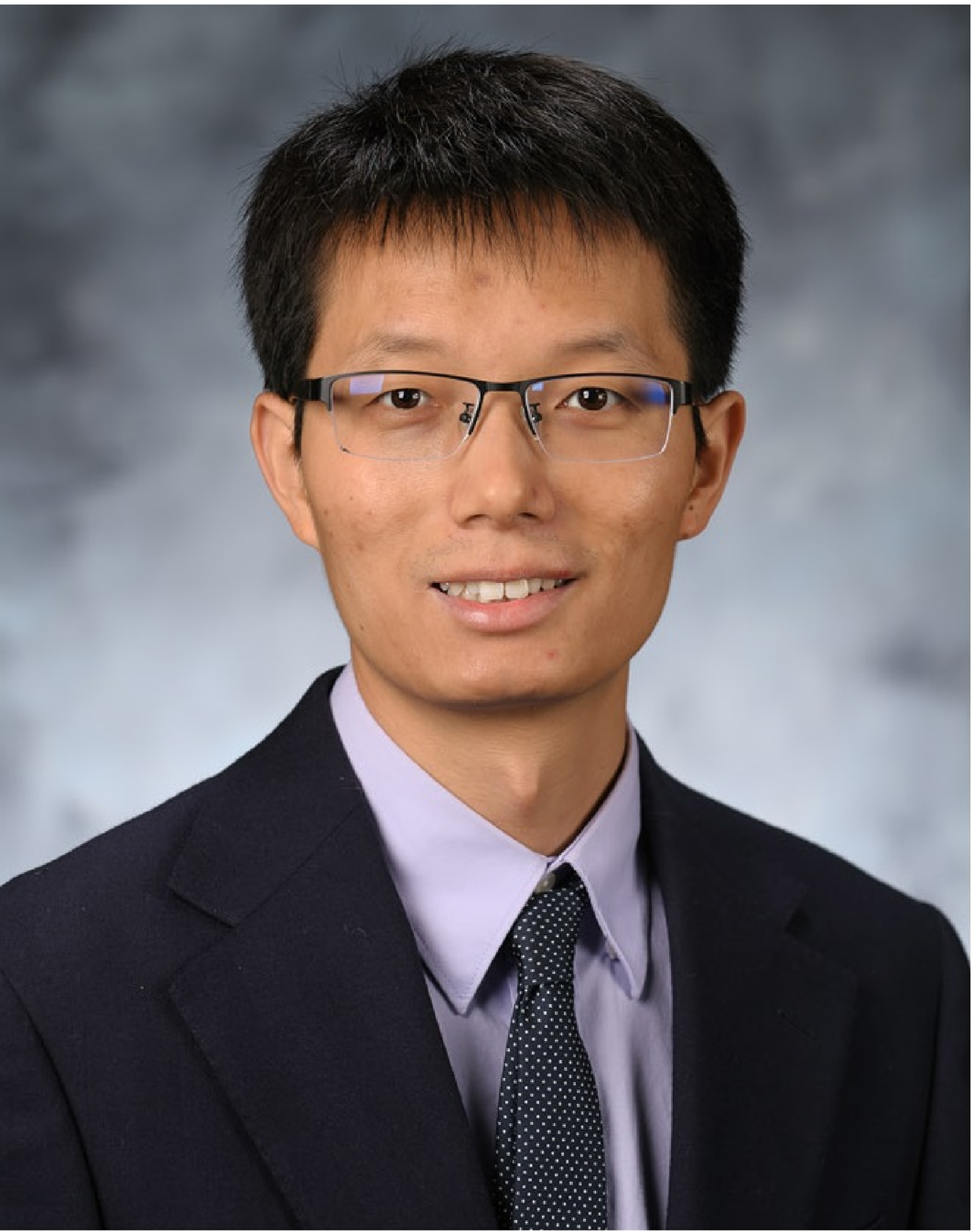}}]{Zheng Zhang} (M'15) received his Ph.D degree in Electrical Engineering and Computer Science from the Massachusetts Institute of Technology (MIT), Cambridge, MA, in 2015. He is an Assistant Professor of Electrical and Computer Engineering with the University of California at Santa Barbara (UCSB), CA. His research interests include uncertainty quantification and tensor computational methods with applications to multi-domain design automation, robust/safe and high-dimensional machine learning and its algorithm/hardware co-design. 

Dr. Zhang received the Best Paper Award of IEEE Transactions on Computer-Aided Design of Integrated Circuits and Systems in 2014, two Best Paper Awards of IEEE Transactions on Components, Packaging and Manufacturing Technology in 2018 and 2020, and three Best Conference Paper Awards (IEEE EPEPS 2018 and 2020, IEEE SPI 2016). His Ph.D. dissertation was recognized by the ACM SIGDA Outstanding Ph.D. Dissertation Award in Electronic Design Automation in 2016, and by the Doctoral Dissertation Seminar Award (i.e., Best Thesis Award) from the Microsystems Technology Laboratory of MIT in 2015. He received the NSF CAREER Award in 2019 and Facebook Research Award in 2020. 
\end{IEEEbiography}

\end{document}